\def\year{2020}\relax
\documentclass[letterpaper]{article}
  \usepackage{aaai20}
  \usepackage{times}
  \usepackage{helvet}
  \usepackage{courier}
  \usepackage[hyphens]{url}
  \usepackage{graphicx}
  \usepackage[title]{appendix}
  \usepackage{graphicx}
\author{
  Dominik Linzner\textsuperscript{\rm 1} and Heinz Koeppl\textsuperscript{\rm 1, \rm 2} \\
 \textsuperscript{\rm 1}Department of Electrical Engineering and Information Technology\\\
 \textsuperscript{\rm 2}Department of Biology\\
  Technische Universit\"at Darmstadt\\
\{dominik.linzner, heinz.koeppl\}@bcs.tu-darmstadt.de
}

\usepackage{bbm}
\usepackage[T1]{fontenc}    
\usepackage{url}            
\usepackage{booktabs}       
\usepackage{amsfonts,amsmath,amssymb,stmaryrd}
\usepackage{nicefrac}       
\usepackage[capitalise]{cleveref}   
\usepackage{microtype}      
\usepackage{algorithm} 
\usepackage{algorithmic} 
\usepackage{nicefrac}       
\usepackage{mathtools}
\usepackage{tabu}
\usepackage{amsthm}

\newtheorem{definition}{Definition}
\theoremstyle{remark}
\newtheorem*{remark}{Remark}
\theoremstyle{theorem}
\newtheorem{theorem}{Theorem}
\theoremstyle{theorem}
\newtheorem{proposition}{Proposition}

\title{A Variational Perturbative Approach to Planning in Graph-based Markov Decision Processes}

\begin{document}

\maketitle

\begin{abstract}
Coordinating multiple interacting agents to achieve a common goal is a difficult task with huge applicability.  This problem remains hard to solve, even when limiting interactions to be mediated via a static interaction-graph.
We present a novel approximate solution method for multi-agent Markov decision problems on graphs, based on variational perturbation theory. We adopt the strategy of planning via inference, which has been explored in various prior works. We employ a non-trivial extension of a novel high-order variational method that allows for approximate inference in large networks and has been shown to surpass the accuracy of existing variational methods. To compare our method to two state-of-the-art methods for multi-agent planning on graphs, we apply the method different standard GMDP problems. We show that in cases, where the goal is encoded as a non-local cost function, our method performs well, while state-of-the-art methods approach the performance of random guess. In a final experiment, we demonstrate that our method brings significant improvement for synchronization tasks.
\end{abstract}

\section{Introduction}
Understanding and designing the behavior of multiple agents interacting through large networks in order to achieve a common goal is a task studied across many fields, such as artificial intelligence \cite{Sigaud2013}, electrical engineering \cite{Tousi2010}, but also economics and biological sciences \cite{Castellano2009} and epidemics \cite{Venkatramanan2018}. Finding optimal policies, e.g., for the distribution of information across a social or communication network, for effective intervention in molecular networks or for vaccinations in order to prevent spreading of diseases are actively discussed problems.
In many of these applications, there exists no unique natural time-scale. In such cases, it is appropriate to reason in continuous-time. The setting of multiple agents on a graph in continuous-time has been previously explored \cite{Kan2008}.

For a Markov decision process (MDP), an optimal policy can be computed in time scaling polynomially in the size of the state and action space using dynamic programming \cite{Puterman2005}. However, in many realistic scenarios, these spaces are high dimensional, e.g., in multi-agent settings \cite{Boutilier1996a}, where the size of the state and action space of the underlying global MDP in general scales exponentially in the number of agents. 
Solving such problems exactly is infeasible for large-scale systems. For this reason, various simplifying assumptions on the structure of MDPs have been proposed. Assuming a factorized state space, a local representation of the transition model and the reward function, decomposing according to a graph-structure, so-called factored MDPs (FMDPs) \cite{Guestrin2001,Boutilier1996b} have been defined. For this model, various approximate solution schemes have been developed \cite{Guestrin2001,Guestrin2003}. 

Graph-based MDPs (GMDPs), as proposed in \cite{Sabbadin2012}, present a subclass of FMDPs, where additionally, agent-wise policies are assumed. We note, that this renders GMDPs equivalent to mMDPs  \cite{Boutilier1996a}, interacting and communicating over a graph-structure. GMDPs can be solved approximately using approximate linear programming \cite{Sabbadin2012}, approximate policy iteration \cite{Sabbadin2012} or approximate value iteration using mean field or cluster variational methods \cite{Cheng2013}. 
Additional simplifying assumptions, such as transition-independence of agents (TI-Dec-MDP) can be made \cite{Sigaud2013}, however reducing the descriptive power of the model. We will thus not compare to such models.

In this work, we propose a novel method for approximate inference and planning for GMDPs inspired by advances in statistical physics. We emphasize that in planning problems \cite{Fleming2006}, system dynamics are known, given a policy. Thus, we do not encounter problems as in reinforcement learning, e.g., as the \emph{exploration-exploitation} dilemma \cite{Puterman2005}. We employ a scheme based on variational perturbation theory \cite{Tanaka1999,Paquet2009,Opper2013,Linzner2018}, which was originally introduced in \cite{Plefka1982}. 

The manuscript is organized as follows: In Section 2, we briefly summarize the connection between variational inference and planning. Here, the main result is that maximization of expected  reward can be coined as maximization of a variational lower bound \cite{Toussaint2006,Furmston2010,Kappen2012}. In Section 3 and 4, we develop an expectation-maximization algorithm to iteratively improve the policy for each agent individually.  Lastly, we perform simulated experiments on several standard planning task and show realistic cases, where current state-of-the-art methods perform similar to random guess, while our method performs well (Section 5). An implementation of our method is available via Git\footnote{https://git.rwth-aachen.de/bcs/vpt-planning}{.}

\section{Background}
\textbf{Continuous-time MDPs on Graphs.}
\begin{figure}[t]
\begin{centering}
		\includegraphics[width=0.95\columnwidth]{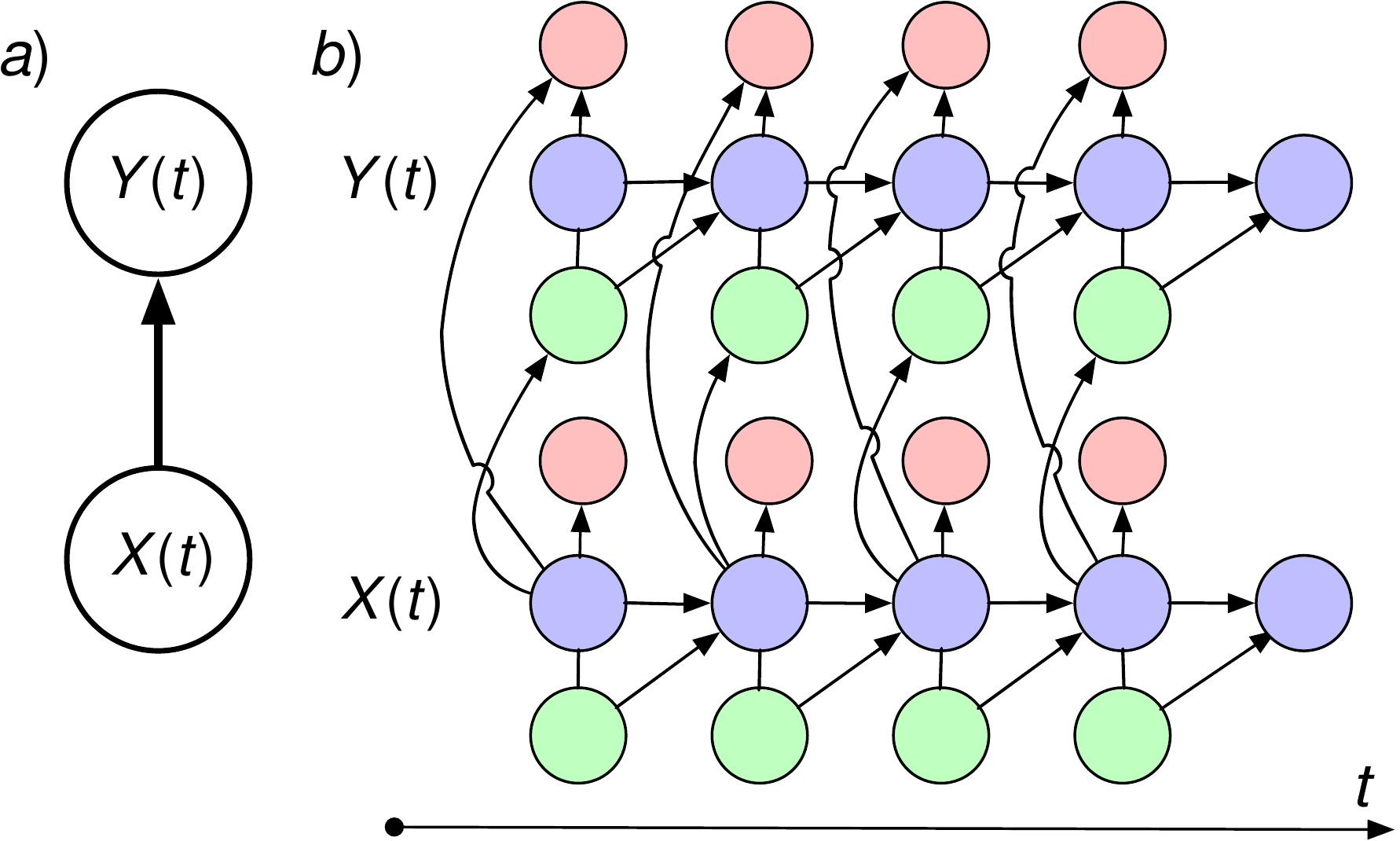}
	\caption{a) A minimal example of a GMDP. The state of agent $Y$ is modulated by its parent $X$. b) The same GMDP unrolled in time as directed graphical model. Agent~$X$ affects agent $Y$'s state (blue) by influencing agent~$Y$'s choice over actions (green) defined by $Y$'s policy. The rewards (red) of agent $Y$  are determined by $Y$'s and $X$'s state. It is also possible to incorporate direct modulation of the transition models by the states of adjacent agents (not displayed for readability).}
	\label{gmdp}
	\end{centering}
\end{figure}
 A MDP models an agent picking actions according to a policy, depending on its current state. Its objective is to minimize its reward, while being subject to some, possibly hostile, environment.  Herein, we define a homogeneous continuous-time MDP by a tupel $(\mathcal{S},\mathcal{A},\mathcal{W},R)$. It defines a two-component Markov process $\{S(t),A(t)\}$ through a transition intensity matrix ${\mathcal{W}}:\mathcal{S}\times\mathcal{S}\times\mathcal{A}\rightarrow\mathbb{R}$ over a countable state space $\mathcal{S}$ and a countable action space $\mathcal{A}$ together with a policy $\pi:\mathcal{A}\times\mathcal{S} \rightarrow[0,1]$. Each state-action pair is mapped to a reward via the \emph{reward function} $R:\mathcal{S}\times \mathcal{A}\rightarrow \mathbb{R}_{-}$. In this work we only consider negative rewards, which poses no restriction as any bounded reward function can be trivially shifted into the negative half-space. For the sake of conciseness, we will often adopt shorthand notations of the type ${p_{t-t'}(s'\mid s,a)\equiv p(S(t)=s'\mid S(t')=s,A(t')=a)}$, with ${s,s'\in\mathcal{S},\,a\in\mathcal{A}}$.
 Given a sequence of actions, the evolution of the MDP can be understood as a usual continuous-time Markov chain (CTMC) with the (infinitesimal) transition probability
 \begin{align*}
p_{{h}}(s'\mid s,a)=\delta_{s,s'}+{h}\: {{\mathcal{W}}}(s'\mid s, a)+o({h}),
 \end{align*}
 for some time-step ${h}$ with  ${\lim_{h\rightarrow 0}o(h)/h=0}$, and $\delta_{s,s'}$ the indicator function. We note, that any intensity matrix ${{\mathcal{W}}}$ fullfils ${{\mathcal{W}}}(s\mid s)=-\sum_{s'\neq s}{{\mathcal{W}}}(s'\mid s)$. 
A multi-agent MDP (mMDP) can be understood as an $N$-component MDP over state- and action-spaces $\mathcal{S}=\bigtimes_{n=1}^N \mathcal{X}_n$, $\mathcal{A}=\bigtimes_{n=1}^N \mathcal{A}_n$, with $\bigtimes$ denoting the Cartesian product, evolving jointly as an MDP. We state explicitly that single component states and actions are entries of the states and actions of the global MDP, i.e. $s=(x_1,\dots,x_N)$ for $s\in\mathcal{S}$ with $x_n\in\mathcal{X}_n$ and $a=(a_1,\dots,a_N)$ for $a\in\mathcal{A}$ with $a_n\in\mathcal{A}_n$ for all $n\in\{1,\dots,N\}$. In this multi-agent setting, each component, referred to as an individual agent, has no direct access to the global state of the system, but can only observe the states of a subset of agents, which we will call its \emph{parent-set}. In the following analysis, we want to restrict ourselves to mMDPs on graphs (GMDPs). 

For GMDPs, the parent configuration can be summarized via a directed graph structure $\mathcal{G}=(\mathcal{V},\mathcal{E})$ encoding the relationship among the agents $\mathcal{V}\equiv\{V_1,\dots,V_N\}$, in this context also referred to as nodes. These are connected via an edge set $\mathcal{E}\subseteq \mathcal{V}\times \mathcal{V}$. The parent-set is then defined as $\mathrm{pa}(n)\equiv\{m \mid  (m,n)\in \mathcal{E}\}$. Conversely, we define the \emph{child-set} $\mathrm{ch}(n)\equiv\{m \mid  (n,m)\in \mathcal{E}\}$. The $n$'th agents process $\{S_n(t),A_n(t)\}$ then depend only on its current state $x_n\in\mathcal{X}_n$, its action  $a_n\in\mathcal{A}_n$ and of all his parents ${U}_n(t)=u_n$ taking values in $\mathcal{U}_n\equiv\bigtimes_{m\in\mathrm{pa}(n)}\mathcal{X}_m$. We display a sketch of a GMDP in Fig. \ref{gmdp}. We note, that cycles in a graphical model as in Fig. \ref{gmdp}(a) are unproblematic, as the corresponding temporally unrolled model, as displayed in Fig. \ref{gmdp}(b), would be acyclic. For a GMDP, the global marginal transition matrix $p_h(s'\mid s,a)$ then factorizes over agents
 \begin{align*}p_h(s'\mid s,a)=\prod_{n=1}^N p_h(y_n\mid x_n,u_n,a_n), \end{align*}
  into local conditional transition probabilities. We define local transition rates ${{w}}_n^u:\mathcal{X}_n\times\mathcal{X}_n\times\mathcal{A}_n\rightarrow\mathbb{R}$ and policies $\pi^u_n:\mathcal{A}_n\times\mathcal{X}_n\rightarrow [0,1]$ for each parent configuration $u\in \mathcal{U}_n$. In the following we write compactly ${{{w}}_n^u(y\mid x, a)\equiv {{w}}_n^u(y_n\mid x_n, a_n)}$ and ${{\pi^u_n(a \mid x)\equiv\pi^u_n(a_n \mid x_n)}}$. Subsequently, we  can express the local conditional transition probabilities as 
\begin{align}\label{eq:local-trans}
&p_h(y_n\mid x_n,u_n,a_n)=\delta_{x,y}+{h}{{w}}_n^u(y_n\mid x_n, a)+o({h}).
\end{align}
We consider the problem of planning in continuous time over a countable state space.
\begin{definition}
\label{eq:Value}
Consider a MDP $(\mathcal{S},\mathcal{A},\mathcal{W},R)$ with initial state $s_0\in\mathcal{S}$ and a policy $\pi$. Then, we can define the (discounted) infinite horizon value function in continuous time as
\begin{align*}
{V^\pi_p(s_0)}=\mathsf{E}_p\left[\int_0^\infty \mathrm{d}t \,{\gamma^t R(S(t),A(t))} \mid S(0)=s_0,\pi\right],
\end{align*}
with $\mathsf{E}_p$ being the expectation with respect to the MDPs path measure $p$.
\end{definition}

We can now cast the planning problem as: for a given initial state $s_0$, find a policy $\pi^*$, such that 
\begin{align}
\pi^*=\underset{\pi}{\arg\max}\{ {V^\pi_p(s_0)} \}.
	\label{eq:SolvMDP}
\end{align}
A common solution strategy for these kinds of problems is to solve the Bellman equation \cite{Puterman2005}. Instead of trying to optimize a Bellman equation, we want to take advantage of the close relationship of planning and inference \cite{Dayan1997,Furmston2010,Toussaint2006,Kappen2012,Levine2013}. In the following, we restrict ourselves to \emph{finite horizon} MDPs, for which the process evolution terminates at time $T$, and later extend to the \emph{infinite horizon} problem.
\newline\newline
\textbf{Finite Horizon Planning via Inference.}
In order to establish the connection between inference and planning we can, following \cite{Dayan1997,Toussaint2006,Levine2013} or similarly \cite{Kappen2012,Furmston2010}, define a boolean auxiliary process $Z(t)$ taking values in $\{0,1\}$, with emission probability
$p(Z(t)=1\mid S(t)=s, A(t)=a)=\exp\{{R(s,a)}\}$. 
We define the finite horizon trajectories ${S_{[0,T]}\equiv\left\{S(\xi)\mid 0\leq\xi \leq T\right\}}$ and ${A_{[0,T]}\equiv\left\{A(\xi)\mid 0\leq\xi \leq T\right\}}$, we can express the reward-optimal posterior process for a given policy~$\pi$ according to Definition \ref{eq:Value} as $p(S_{[0,T]},A_{[0,T]}\mid Z_{[0,T]}=1,\pi,s_0)$, with $ Z_{[0,T]}=1$ meaning that $Z(\xi)=1$ for $0\leq\xi \leq T$.
We consider the Kullback--Leibler (KL) divergence between the posterior ${p(S_{[0,T]},A_{[0,T]}\mid Z_{[0,T]}=1,\pi,s_0)}$ and a variational measure $q(S_{[0,T]},A_{[0,T]}\mid \pi,s_0)$ induced by a time-inhomogeneous MDP with the same policy as $p$ (in supplementary B, we show that the KL-divergences between two continuous-time MDPs with different policies diverges). We arrive at a lower bound for the marginal log-likelihood in the finite horizon case
\begin{align}
\label{eq:lb}
	&\ln p(Z_{[0,T]} =1\mid \pi,s_0)\geq \mathcal{{F}}[q,\pi]+V^\pi_q(s_0),
	\end{align}
	\begin{align}
	\nonumber&\mathcal{{F}}[q,\pi]\equiv \\
	\nonumber&-D_{KL}[q(S_{[0,T]},A_{[0,T]}\mid \pi, s_0)| |p(S_{[0,T]},A_{[0,T]}\mid \pi,s_0)],
\end{align}
with the variational lower bound $\mathcal{{F}}[q,\pi]$.  The full derivation and structure of \eqref{eq:lb} can be found in supplementary A and B. When performing exact inference, meaning that $q(S_{[0,T]},A_{[0,T]}\mid \pi,s_0)=p(S_{[0,T]},A_{[0,T]}\mid \pi,s_0)$, lower bound and log-likelihood coincide and the maximization of the value function as in Definition \ref{eq:Value} corresponds to a maximization of the log-likelihood w.r.t the policy
\begin{align*}
	\underset{\pi}{\arg\max}\{ V_q^\pi(s_0) \}=\underset{\pi}{\arg}\max \{ \ln p(Z_{[0,T]} =1\mid \pi,s_0)\},
\end{align*}
establishing the connection between planning and inference.
When performing approximate inference, we can iteratively maximize the lower bound with respect to $q$ and thereby approximate the log-likelihood, following a maximization with respect to $\pi$. This is the expectation-maximization algorithm, which has been previously applied to policy optimization \cite{Toussaint2006,Levine2013}. 
\newline\newline
\textbf{Infinite Horizon Planning via Inference.}
The same framework as above can be used in order to solve (discounted) infinite horizon problems. Following \cite{Toussaint2006}, this can be achieved by introducing a prior over horizons $p(T)$. As a derivation in continuous-time is missing in literature, we provide it in supplementary C. By choosing $p(T)=\ln(\gamma)\gamma^T$, one recovers exponential discounting with discount factor $\gamma$.

\section{Variational Perturbation Theory for GMDPs}
Calculating a variational lower bound exactly is in general intractable for interacting systems. This is often circumvented by assuming a factorized proposal distribution $q(x)=\prod_i q_i(x_i)$, which corresponds to the \emph{naive mean-field} approximation. Variational perturbation theory (VPT) offers a different approach. Here, the similarity measure (the KL-divergence) itself is approximated via a series expansion \cite{Tanaka1999}. A prominent example of this approach is Plefka's expansion \cite{Plefka1982,Bachschmid-Romano2016}. The central assumption is that variables are only weakly coupled, i.e. the interaction of variables is scaled in some small perturbation parameter $\varepsilon$. In this case, the objective is to find an expansion of the KL-divergence in orders of the interaction parameter $\varepsilon$:
$
{\mathcal{F}[q,\pi]=\mathcal{F}^{(0)}[q,\pi]+\varepsilon \mathcal{F}^{(1)}[q,\pi]+\dots.}
$
This approximate variational lower bound is then maximized with respect to $q$. We note that $\mathcal{F}[q,\pi]$, like in the case of cluster variational methods \cite{yedidia2000} (CVMs), no longer provides a lower bound but only an approximation. However, in contrast to CVMs (which can be used construct similar approximate KL--divergences \cite{Vazquez2017}), variational perturbation theory yields a controlled approximation in the perturbation parameter $\varepsilon$. 
\newline
\newline
\textbf{Weak Coupling Expansion.}
In the following, we want to briefly recapitulate and extend the weak coupling expansion for the lower bound in \eqref{eq:lb}, as derived in \cite{Linzner2018} in the context of factorized CTMCs, to (discounted) infinite horizon GMDPs.
For this we notice, that the lower bound $\mathcal{F}[q,\pi]$ decomposes over time
\begin{align*}
&\mathcal{F}[q,\pi]=\lim_{h\rightarrow 0}\frac{1}{h}\int_0^\infty \mathrm{d}t \: f^h_t[q,\pi],
\\& f^h_t[q,\pi]  =\sum_{s,s',a}\pi(a\mid s)q(s;t)q_h(s'\mid s,a)\ln \frac{p_{{h}}(s'\mid s,a)}{q_h(s'\mid s,a)},
\end{align*}
where we introduced the shorthands for the marginals $q(s;t)\equiv q(S(t)=s)$ and the infinitesimal transition matrix ${q_h(s'\mid s,a)\equiv q(S(t+h)=s'\mid S(t)=s,A(t)=a)}$ of the variational process $q$, for notational convenience. 

For a weak coupling expansion, we decompose the node-wise transition probability into an uncoupled part, given by averaging over parents $p_h({y}_n \mid x_n, a_n) = \mathsf{E}[ p_h({y}_n\mid x_n, u_n,a_n) \mid x_n]$ and a deviation around it, defined as $g({y}_n,x_n,u_n,a_n) \equiv p_h({y}_n \mid x_n, u_n,a_n) - p_h({y}_n\mid x_n,a_n)$. Following standard mean-field procedure, we extract a scale parameter $g({y}_n,x_n,u_n,a_n)=\varepsilon\tilde{g}({y}_n,x_n,u_n,a_n)$, with $\tilde{g}({y}_n,x_n,u_n,a_n)$ having the same magnitude as the uncoupled part.
This allows to rewrite the transition matrix
\begin{align}\label{eq:expansion}
p_h({y}_n \mid x_n, u_n,a_n) = p_h({y}_n \mid x_n,a_n) + \varepsilon\tilde{g}({y}_n,x_n,u_n,a_n). 
\end{align}
We emphasize, that this procedure is generic and can be performed for any transition probability. This motivates the weak-coupling expansion on which the results in this manuscript are build upon, for which we define the shorthand $q(y_n,x_n,u_n,a_n;t)\equiv q(S_n(t+h)={y}_n,S_n(t) = x_n ,U_n(t) = u_n,A_n(t)=a_n)$.
\begin{theorem}[Weak coupling expansion for GMDPs]
\label{plefka}
The time point wise lower bound $f^h_t[q,\pi]$ admits an expansion in $\varepsilon$, as given in \eqref{eq:expansion}, into node-wise terms $f_{t,n}[q,\pi]$ 
\begin{align*}
&f^h_t[q,\pi]=\sum_{n=1}^N f^h_{t,n}[q,\pi_n]+o(\varepsilon),\\
&f^h_{t,n}[q,\pi_n]=\sum_{x_n,{y}_n,u_n,a_n}\pi_n^{u_n}(a_n\mid x_n)q_t({y}_n,x_n, u_n)\\
&\quad\quad\quad\quad\times\ln \frac{p_{{h}}({y}_n\mid x_n, u_n,a_n)}{ q_h({y}_n\mid x_n,u_n,a_n)}.
\end{align*} 
\end{theorem}
The proof of this theorem is along the same lines as in \cite{Linzner2018}.
\newline
\newline
\textbf{Weak Coupling Expansion for GMDPs in Continuous Time.}
In order to derive the approximate variational lower bound in continuous time for a GMDP, we define variational marginal rates 
\begin{align*}
\tau^{u_n}_n(x_n,{y}_n, a_n;t)\equiv \lim_{{h}\rightarrow 0}\frac{q({y}_n,x_n,a_n,u_n;t)}{{h}}\quad \text{ for } x_n\neq y_n
\end{align*}
and $\tau^{u_n}_n(x_n,x_n,a_n;t)=-\sum_{{y}\neq x} \tau^{u_n}_n(x_n,{y}_n,a_n;t)$ but will from now on use the redefinition ${x\equiv x_n},{{y}\equiv {y}_n}$,${a\equiv a_n}$, ${u\equiv u_n}$ for these objects, in order to avoid clutter. We further make use a mean-field assumption ${q({y}_n,x_n,u_n;t)=q_h({y}_n\mid x_n,u_n;t)q_n(x;t)q_n^u(t)}$, with the shorthand ${q}_n^u(t)\equiv\prod_{j\in\mathrm{par(n)}}{q}_n(u_j,t)$, assuming factorization of the marginals. We emphasize, that in contrast to naive mean-field \cite{Opper2008,Cohn2010}, we only have to assume a factorization of these marginals, but keep the dependency on the parents in the rates $\tau^{u}_n(x,{y},a;t)$.
Together with the normalization constraint, this defines an expansion of the proposal transition probability in time-steps of ${h}$:
$q({y}_n,x_n,u_n,a_n;t)=\delta_{x,{y}}{q}_n(x;t){q}_n^u(t)\pi_n^u(a\mid x)+{h}\tau^u_n(x,{y},a;t)+{o}({h})$.
The proposal transition probability defines an inhomogeneous master equation 
\begin{align}
\dot{q}_n(x;t)=\sum_{{y}\neq x,u,a}[\tau^u_n({y},x,a;t)-\tau^u_n(x,{y},a;t)].\label{eq:continuity}
\end{align}
In order for $q$ to describe a probability distribution, this constraint has to be enforced at all times.
\begin{proposition}
\label{eq:plefka-lb}
The variational lower bound of a GMDP has an expansion into agent-wise terms in the perturbation parameter $\varepsilon$
\begin{align}\label{eq:vpt-lb}
&\nonumber\mathcal{{F}}[q, \pi]= \mathcal{{F}}_{\mathrm{VPT}}[q, \pi]+o(\varepsilon)\\
&\mathcal{{F}}_{\mathrm{VPT}}[q, \pi]= \sum_{n=1}^N\int_0^\infty \mathrm{d}t\, d_\gamma(t)\left\{H_n(t)+E_n(t)\right\},\\
&\nonumber H_n(t)=\sum_{{y},x\neq y,u,a}\tau^u_n({y},x,a;t)\ln\left\{\frac{\tau^u_n({y},x,a;t)}{q_n(x;t)q_n^{u}(t)}-1\right\},\\
&\nonumber E_n(t)=\sum_{{y},x\neq y,u,a}\left\{ q_n(x;t)q_n^{u}(t)\pi^u_n(a\mid x){w}^u_n({y},x|a)\right.\\
&\nonumber \left.\quad\quad\quad\quad+\tau_n^u(y,x,a;t)\ln\left[{w}_n^u(y,x|a)\pi^u_n(a\mid x)\right]\right\} ,
\end{align}
with  the discounting function $d_\gamma(t)\equiv 1-\int_0^t \mathrm{d}T\, p(T)$.
\end{proposition}
\begin{proof}
We proof our proposition by inserting the marginals into the expansion of Theorem \ref{plefka}. We insert the expression of the conditional transition matrix \eqref{eq:local-trans}. Subsequently, we perform ${h} \rightarrow 0$. We arrive at the approximate lower bound of a GMDP.  The discounting function follows from Fubini's theorem. For a detailed derivation, see supplementary D.
\end{proof}
By minimizing this functional, while fulfilling continuity, we can derive approximate dynamic equations corresponding to the stationary solutions of the Lagrangian 
\begin{align}
\label{eq:euler--lagrange}
&\mathcal{{L}}[q,\pi,\eta]=\mathcal{{F}}_{\mathrm{VPT}}[q,\pi]+\mathcal{C}[q,\eta]+V^\pi_q(s_0),
\end{align} 
with $\mathcal{C}[q,\eta]$ being the constrain enforcing \eqref{eq:continuity} (see supplementary E) and Lagrange multipliers $\eta_n(t)$. 

\section{Approximate Inference}
 \begin{algorithm}[t]
   \caption{Stationary points of Euler--Lagrange equation}
   \label{alg:euler-lagrange}
\begin{algorithmic}[1]
   \STATE {\bfseries Input:} Initial trajectories ${q}_n(x;t) \forall t\in[0,T]$ obeying normalization, boundary conditions $q(x;0)$ and $\rho(x;T)$, reward function $R(s,a)$.
   \REPEAT
   \FORALL {$n\in \{1,\dots,N\}$ }
   \STATE Update $\rho_n(x;t)$ by backward propagation \eqref{eq:dyn-marg-II}.
    \STATE  Update ${q}_n(x;t)$ by forward propagation using \eqref{eq:dyn-marg-I} given $\rho_n(x;t)$.
    \ENDFOR
   \UNTIL{Convergence \eqref{eq:vpt-lb}}
   \STATE {\bfseries Output:} Set of ${q}_n(x;t)$ and $\rho_n(x;t)$.
\end{algorithmic}
\end{algorithm}
We finally derive approximate dynamics of the GMDP as stationary points of the Lagrangian, satisfying the Euler--Lagrange equation. These are the key equations that enable us to perform scalable approximate inference for large GMDPs. 
\begin{proposition}
\label{euler--lagrange--dyn}
We define the agent-wise expectation $\mathsf{E}^{\pi}_n[f(x)]\equiv\sum_{u,a} \pi_n^u(a\mid x){q}_n^u(t) f(a,u,x)$.
The stationary points of the Lagrangian \eqref{eq:euler--lagrange} are given by the set of ordinary differential equations for every component $n\in \{1,\dots,N\}$
\begin{align}\label{eq:dyn-marg-I}
&\dot{{q}}_{n}(t)={q}_n(t)\Omega_n(t)
 \end{align}
 \begin{align}\label{eq:dyn-marg-II}
&\dot{{\rho}}_{n}(t)=\left\{\Omega_n(t)+\Theta_n(t)+\Psi_n(t)\right\}{\rho}_{n}(t)
 \end{align}
 with
 \begin{align*}
&\Omega_n(x,y;t)\equiv{\mathsf{E}^{\pi}_{n}[{{w}}^u_n(x,y\mid a)]}\frac{{\rho_{n}(y;t)}}{\rho_{n}(x;t)}\\
&\Theta_n(x,y;t)\equiv\delta_{x,y}\left(\mathsf{E}^{\pi}_{n}[{R}^u_n(x,a)]+\ln{{\rho_n(x;t)}}\frac{\partial_t d_\gamma(t)}{d_\gamma(t)}\right)
\end{align*}
with $\Psi_{n}(t)$ as given in the supplementary and ${R(s,a)=\sum_{n=1}^N R_n^u(x,a)}$. We note, that for exponential discounting $\frac{\partial_t d_\gamma(t)}{d_\gamma(t)}=\ln\gamma$.

\end{proposition}
\begin{proof}
Differentiating $\mathcal{{L}}$ with respect to ${q}_n(x;t)$, its time-derivative $\dot{q}_n(x;t)$, $\tau^u_n(x,y,a;t)$ and the Lagrange multiplier $\eta_n(x;t)$ yield a closed set of coupled 
ODEs for the posterior process of the marginal distributions ${q}_n(x;t)$ and transformed Lagrange multipliers $\rho_n(x;t)\equiv\exp(\eta_n(x;t)/d_\gamma(t))$, eliminating $\tau^u_n(x,y,a;t)$.
For more details, we refer the reader to supplementary E.
\end{proof} 
Although, the restriction on the reward function to decompose into local terms is not necessary, we will assume it for readability.
The coupled set of ODEs can be solved iteratively as a fixed-point procedure in the same manner as in previous works \cite{Opper2008} in a forward-backward procedure (see Algorithm \ref{alg:euler-lagrange}). Because we only need to solve $2N$ ODEs to approximate the dynamics of an $N$-agent system, we recover a linear complexity in the number of agents, rendering our method scalable. 

We require boundary conditions for the evolution interval in order to determine a unique solution to the set of equations in Proposition \ref{euler--lagrange--dyn}.
We thus set ${q}_n(x;0)=\delta_{x,x_0}$ to the desired initial state $x_0$ and $\rho_n(x;t)=1$ for free evolution of the system. 
We note that while we do not consider time-dependent reward in general, our method is capable of doing so. We use this in the following control setting:
in control scenarios, a deterministic \emph{goal} state of the system is often desired \cite{Kappen2012}. In this case, we can put infinite reward on the goal state $x_T$ at the boundary $T$. We then recover the terminal condition $\rho_n(x;t)=\delta_{x,x_T}$.
 By setting the reward-dependent terms in Proposition \ref{euler--lagrange--dyn} to zero, we can evaluate the prior dynamics of the system given a policy. We will use this to evaluate Definition \ref{eq:Value} approximately.
\newline
\newline
\textbf{Expectation-Maximization for GMDPs.}
By examining the approximate lower bound of the value function, one notices that it decomposes over local agent-wise value functions, conditioned on its parents.
\begin{remark}
\label{rem:Local value functions}
The marginal log-likelihood of a GMDP has an approximate agent-wise decomposition
\begin{align}\label{eq:vpt-lb-val}
\ln p(Z_{[0,T]}=1\mid \pi)\geq \sum_{n=1}^N \mathcal{F}^n_{\mathrm{VPT}}[q,\pi]+V^\pi_q(s_0)+o(\varepsilon),
\end{align}
where the $\mathcal{F}^n_{\mathrm{VPT}}[q,\pi]$ are given by Proposition \ref{eq:plefka-lb}. 
\end{remark}

Because of this, the global marginal log-likelihood can be maximized by locally maximizing local lower bounds of the individual agents with respect to local policies $\pi_n$. 
Given the dynamic equations from Proposition \ref{euler--lagrange--dyn}, we now devise a strategy for scalable planning for GMDPs. For this we notice, that the solutions of these equations maximize the lower bound, thereby providing an approximation to the marginal log-likelihood. Because of \eqref{eq:vpt-lb-val}, we can maximize this object as well with respect to the policies $\pi_n$ for each agent individually. Thus the complexity of our optimization scales linearly in the number of components. Given this maximizer, we again evaluate the dynamic equations. We do this repeatedly until convergence, thereby implementing an expectation-maximization (EM) algorithm. This strategy is summarized in Algorithm \ref{alg:RL}. 
We note that the resulting policy is probabilistic, but a MAP-deterministic policy can be constructed.
 \begin{algorithm}[t]
   \caption{Expectation-Maximization for Planning}
   \label{alg:RL}
\begin{algorithmic}[1]
   \STATE {\bfseries Input:} Initial trajectories ${q}_n(x;t)\forall t\in[0,T]$ obeying normalization, boundary conditions $q(x;0)$ and $\rho(x;T)$, reward function $R(s,a)$, initial policy $\pi^{(0)}$.
   \STATE Set  $i=0$
   \REPEAT
   \STATE Solve Euler-Lagrange equations given $\pi^{(i)}$ using Algorithm \ref{alg:euler-lagrange}.
   \FORALL {$n\in \{1,\dots,N\}$ }
   \STATE Maximize \eqref{eq:vpt-lb-val} with respect to $\pi_n$'s.
   \STATE Set maximizer $\pi_n^{(i+1)}=\pi_n^*$.
      \ENDFOR
   \STATE $i\rightarrow i+1$
   \UNTIL{Convergence of \eqref{eq:vpt-lb-val}}
   \STATE {\bfseries Output:} Optimal policy $\pi^*$.
\end{algorithmic}
\end{algorithm}

\section{Experiments}
We evaluate the performance of our method on real-world problem settings against two existing state-of-the-art methods for GMDPs on different network topologies. One method is based on policy iteration in mean-field approximation (API) \cite{Sabbadin2012}, the other on approximate linear programming (ALP) \cite{Guestrin2001}. Both algorithms have been developed and implemented in the GMDPtoolbox \cite{Cros2017}. For small problems, we compare the performance of all algorithms to the exact solution. To ensure a correct evaluation, we first construct the GMDP problem and then transform it to the corresponding MDP problem by a built-in function in the GMDPtoolbox, in order to recover the exact solution. For small problems, we finally perform exact policy evaluation using this MDP.

As competing methods are implemented in discrete-time, we have to pass them an equivalent discrete-time version of the continuous-time problem via uniformization \cite{Kan2008}. For this we generate transformed rewards and transition matrices
\begin{align*}
\tilde{R}^u_n(x,a)\equiv&\frac{{w}_n^u(x\mid x,a)-\ln\gamma}{\kappa-\ln\gamma}R^u_n(x,a),\\
p_{1/\kappa}(y_n\mid a_n,u_n,x_n)\equiv&
\begin{cases}
\begin{array}{c}
{w}_{n}^{u}(y\mid x,a),\\
\kappa+{w}_{n}^{u}(x\mid x,a),
\end{array} & \begin{array}{c}
x\neq y\\
\mathrm{{else}}
\end{array}\end{cases}
\end{align*}
for some ${\kappa\geq|{w}_n^u(x\mid x,a)|}$.

GMDPs have previously been applied to a variety of problems as in agriculture, forest management \cite{Peyrard2007,Sabbadin2012}, socio-physics \cite{Castellano2009,Yang2018} and caching networks \cite{Rezaei2018}, to name a few. In the following we want to benchmark our method on those problem sets. We want compare to the exact solution, thus the network considered is a small regular $2\times3$ grid, with nearest-neighbour bi-directional couplings, unless specified otherwise. In the end, we demonstrate scalability on a larger $5\times 5 $ grid in a synchronization task experiment. We denote the policies returned by the different methods with $\pi_\mathrm{ALP} $ for ALP,  $\pi_\mathrm{API} $ for API,  $\pi_\mathrm{RND} $ for a random policy and  $\pi_\mathrm{VPT} $ for our method VPT. For all experiments, we set the discount factor to $\gamma=0.9$ and the atomic reward $r=1$. As a metric for performance, we calculate the relative deviation $d_r(\pi)\equiv (V^{\pi^*}-V^{\pi})/V^{\pi^*}$ (with ${\pi^*}$ being the exact optimal policy) in percent for the crop and forest planning problem, and the $95\%$ interval of total deviation for the opinion dynamics model.
\begin{table}[t]
\caption{Results of disease control problem. We give the relative deviation $d_r[\%]$ of the values returned by different methods from the exact optimal values.}
\label{tab:results-crop}
\begin{align*} 
\begin{centering}   
\begin{tabular}{c|c|c|c|c}
\toprule
\hline
$(\mu,\nu)$&${\pi_\mathrm{VPT}}$&${\pi_\mathrm{ALP}}$&${\pi_\mathrm{API}}$&${\pi_\mathrm{RND}}$\\\hline
$(0.3,0.3)$&$\mathbf{0}$&$61$&$\mathbf{0}$&$200$\\\hline
$(0.6,0.3)$&$\mathbf{0}$&$60$&$\mathbf{0}$&$292$\\\hline
$(0.9,0.3)$&$\mathbf{0}$&$59$&$\mathbf{0}$&$270$\\\hline
$(0.3,0.6)$&$\mathbf{0}$&$61$&$\mathbf{0}$&$281$\\\hline
$(0.6,0.6)$&$\mathbf{0}$&$60$&$\mathbf{0}$&$354$\\\hline
$(0.9,0.6)$&$\mathbf{0}$&$59$&$\mathbf{0}$&$399$\\\hline
$(0.3,0.9)$&$\mathbf{0}$&$61$&$\mathbf{0}$&$390$\\\hline
$(0.6,0.9)$&$\mathbf{0}$&$60$&$\mathbf{0}$&$344$\\\hline
$(0.9,0.9)$&$\mathbf{0}$&$59$&$\mathbf{0}$&$352$\\\hline
\bottomrule
\end{tabular}
\end{centering}   
\end{align*} 
\end{table}
\newline\newline
\textbf{Disease Control.}
First, we apply our method to the task of disease control, originally posed for  crop fields  \cite{Sabbadin2012}. Each crop is in either of two states --  susceptible or infected ($\mathcal{X}=\{1,2\}$). The rate $\alpha({u})=1+\frac{1}{2}(1-(1-\mu)^{|{u}|})$, with which a susceptible crop is infected, is proportional to the number of its infected neighbours, which we denote by ${|u|}$. The recovery rate is assumed to be constant $\nu$. The planner has to decide between two local actions for each crop -- either to harvest or to leave it fallow and treat it ($\mathcal{A}=\{1,2\}$). Below, we summarize the transition model: 
\begin{align*}  
\begin{centering}   
\begin{tabular}{c | c c | c c}
\toprule
\hline
&$a=1$&&$a=2$&\\
\hline
${w}_n^u$&$x=1$&$=2$&$x=1$&$=2$\\ \hline
$y=1$&$-\alpha({u})$&$\alpha({u})$&0&0\\
$y=2$&0&0&$\nu$&$-\nu$\\ \hline
\bottomrule
\end{tabular}
\end{centering}   
\end{align*} 
The reward model is: 
\begin{align*}  
\begin{centering}   
\begin{tabular}{c | c c }
\toprule
\hline
$R_n^u$&$x=1$&$=2$\\ \hline
$a=1$&$0$&$0$\\
$a=2$&$r$&$r/2$\\ \hline
\bottomrule
\end{tabular}
\end{centering}   
\end{align*}
In Table \ref{tab:results-crop}, we display the results of this experiment for different parameters. We find that API and VPT perform equally well in this problem.
\begin{table}[t]
\caption{Results of forest management problem. We give the relative deviation $d_r[\%]$ of the values returned by different methods from the exact optimal values.}
\label{tab:results-forest}
\begin{align*} 
\begin{centering}   
\begin{tabular}{c|c|c|c|c}
\toprule
\hline
$(\mu,\nu)$&${\pi_\mathrm{VPT}}$&${\pi_\mathrm{ALP}}$&${\pi_\mathrm{API}}$&${\pi_\mathrm{RND}}$\\\hline
$(0.3,0.3)$&$\mathbf{0}$&$1$&$1$&$12$\\\hline
$(0.6,0.3)$&$\mathbf{0}$&$8$&$1$&$12$\\\hline
$(0.9,0.3)$&$\mathbf{0}$&$12$&$1$&$11$\\\hline
$(0.3,0.6)$&$\mathbf{1}$&$27$&$12$&$23$\\\hline
$(0.6,0.6)$&$\mathbf{1}$&$26$&$11$&$22$\\\hline
$(0.9,0.6)$&$\mathbf{1}$&$26$&$12$&$21$\\\hline
$(0.3,0.9)$&$\mathbf{9}$&$58$&$25$&$48$\\\hline
$(0.6,0.9)$&$\mathbf{9}$&$58$&$27$&$50$\\\hline
$(0.9,0.9)$&$\mathbf{10}$&$58$&$38$&$52$\\\hline
\bottomrule
\end{tabular}
\end{centering}   
\end{align*} 
\end{table}
\newline\newline
\textbf{Forest Management.}
We consider the forest management problem as in \cite{Sabbadin2012}. Here, each node has multiple states dependent of each trees age and whether it is damaged by wind or not. A tree can either age or become damaged over time. In a simplified scenario, we are going to assume, that a tree can either be grown -- or not -- or damaged ($\mathcal{X}=\{1,2,3\}$). As trees can shield one-another against wind-damage, this rate $\alpha({u})=1+\frac{1}{2}(1-(1-\mu)^{-{|u|}})$ depends on the number of grown trees $|u|$. The planner has, again, two actions -- either to harvest and cut down the tree or to leave it ($\mathcal{A}=\{1,2\}$). The transition model is summarized below: 
\begin{align*}  
\begin{centering}   
\begin{tabular}{c| c c  c | c c c}
\toprule
 \hline
&$a=1$& & &$a=2$& &\\
\hline
${w}_n^u$&$x=1$&$=2$&$=3$&$x=1$&$=2$&$=3$\\ \hline
$y=1$&$-\nu$&$\nu$&0 &0 &0&0\\
$y=2$&0&$-\alpha({u})$&$\alpha({u})$&$1$ & $-1$&$0$\\ 
$y=3$&0&0&$0$&$1$ &0 &$-1$\\ \hline
\bottomrule
\end{tabular}
\end{centering}   
\end{align*} 
As yield depends on having neighbours for various reasons, the reward function in \cite{Sabbadin2012} has a non-local form. We consider reward functions as: 
\begin{align*}  
\begin{centering}   
\begin{tabular}{c| c c c }
\toprule
 \hline
$R_n^u$&$x=1$&$=2$&$=3$\\ \hline
$a=1$&$0$&$0$&0 \\
$a=2$&0&$r-|u|$&$\frac{r-|u|}{2}$\\ \hline
\bottomrule
\end{tabular}
\end{centering}   
\end{align*} 
The results of this experiment are displayed in Table \ref{tab:results-forest}, where we give the relative deviation in percent between the optimal and the policies returned from the different methods. We find that for all parameters, our method performs significantly better than other methods.
\newline\newline
\textbf{Opinion Dynamics.} In this experiment we test the performance of our method on the seminal Ising model, which has, among others, applications in socio-physics \cite{Castellano2009} to model opinion dynamics, swarming \cite{Sosic2017}, or as a benchmark for multi-agent reinforcement learning \cite{Yang2018}.  In the Ising model, each node is in either of two states $\mathcal{X}=\{-1,1\}$ and the reward function takes the form 
\begin{align}\label{ising-reward}
  R(s,a)=\sum_{n=1}^N x_n\left\{J_n+\sum_{k\in\mathrm{par}(n)}J_{n,k}x_k\right\}. \end{align}
In the following, we want to consider random reward functions, where couplings are drawn from gaussians $J_n\sim \mathcal{N}(0,\mu)$ and $J_{n,k}\sim \mathcal{N}(0,\nu)$. Further, we model the transition rates according to opinion dynamics (voter model) \cite{Castellano2009} $\alpha(u)=\frac{1}{2}\left[1+ \tanh\left( |{u}|\right)\right]$ and $\beta(u)=\frac{1}{2}\left[1- \tanh\left( |{u}|\right)\right]$, with $|u|$, being the sum of the sequence $u$, see below:
\begin{align*}  
\begin{centering}   
\begin{tabular}{c | c c | c c}
\toprule
\hline
&$a=1$&&$a=2$&\\
\hline
${w}_n^u$&$x=-1$&$=1$&$x=-1$&$=1$\\ \hline
$y=-1$&$-\alpha(u)$&$\alpha({u})$&$-\beta({u})$&$\beta({u})$\\
$y=1$&$\beta(u)$&$-\beta({u})$&$\alpha({u})$&$-\alpha({u})$\\ \hline
\bottomrule
\end{tabular}
\end{centering}   
\end{align*} 
The results for an ensemble 20 random reward functions displayed in \cref{voter-small}. Again, we find that our method performs best in all tested parameter regimes, while in some cases RND achieves a higher value than API and ALP.
\begin{table}[t]
\caption{Results of voter model for an ensemble of 20 random reward functions. We give the $95\%$ interval of the deviation of the achieved values returned by different methods from the exact optimal value.}\label{voter-small}
\begin{align*}  
\begin{centering}
\begin{tabular}{c|c|c|c|c}
\toprule
\hline
$(\mu,\nu)$&${\pi_\mathrm{VPT}}$&${\pi_\mathrm{ALP}}$&${\pi_\mathrm{API}}$&${\pi_\mathrm{RND}}$\\\hline
$(0.1,0.0)$&$\mathbf{0.0}$&$\mathbf{0.0}$&$\mathbf{0.0}$&$1.9$\\\hline
$(0.2,0.0)$&$\mathbf{0.4}$&$\mathbf{0.4}$&$\mathbf{0.4}$&$5.6$\\\hline
$(0.0,0.1)$&$\mathbf{0.8}$&$5.5$&$5.2$&$4.7$\\\hline
$(0.1,0.1)$&$\mathbf{1.8}$&$5.6$&$5.7$&$6.4$\\\hline
$(0.2,0.1)$&$\mathbf{0.4}$&$1.2$&$1.8$&$5.6$\\\hline
$(0.0,0.2)$&$\mathbf{0.0}$&$3.5$&$4.2$&$4.7$\\\hline
$(0.1,0.2)$&$\mathbf{0.1}$&$4.4$&$8.8$&$9.1$\\\hline
$(0.2,0.2)$&$\mathbf{0.0}$&$3.0$&$4.7$&$10.4$\\\hline
\bottomrule
\end{tabular}
\end{centering}
\end{align*}  
\end{table}
\newline\newline
\textbf{Synchronization of Agents.}
In a final experiment, we want to compare the performance of methods in a synchronization task. We consider a regular grid of $5\times5$ agents. We encode a synchronization goal by reward function as in \eqref{ising-reward} with $J_n=0$ and $J_{n,k}=-1$. The reward function takes the from of an order parameter $\mathcal{R}(s)=\sum_{i,j\in\mathrm{par}(i)}\delta_{x_i\neq x_j}$, which measures anti-parallel alignment between neighbouring agents. Each agents transition model is local:
\begin{align*}  
\begin{centering}   
\begin{tabular}{c | c c | c c}
\toprule
\hline
&$a=1$&&$a=2$&\\
\hline
${w}_n^u$&$x=-1$&$=1$&$x=-1$&$=1$\\ \hline
$y=-1$&$-0.9$&$0.9$&$-0.1$&$0.1$\\
$y=1$&$0.1$&$-0.1$&$0.9$&$-0.9$\\ \hline
\bottomrule
\end{tabular}
\end{centering}   
\end{align*} 
We display $\mathcal{R}(s)$ over time for different methods in Figure \ref{sync} (LP returned the same policy as MF). For evaluation, we simulated each trained model using Gillespie sampling. 
\begin{figure}[t]
\begin{centering}
\includegraphics[width=0.95\columnwidth]{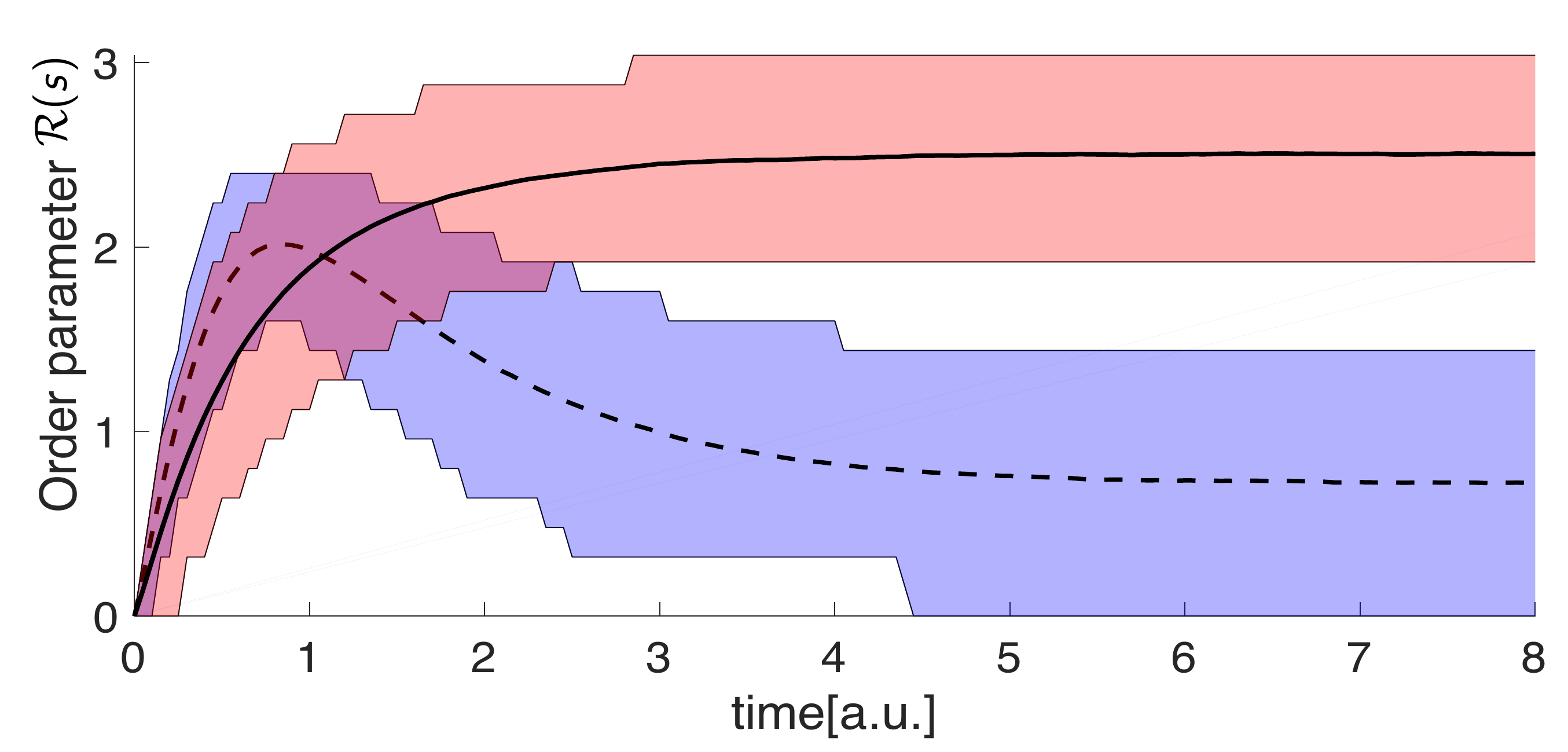}
	\caption{Results of the synchronization task. We track the mean order parameter over time under the VPT (red) and MF (blue-dashed) policy. Areas denote $90\%$ percent of variance.}
	\label{sync}
	\end{centering}
\end{figure}

\section{Conclusion}
We proposed a new method to conduct planning on large scale GMDPs based on variational perturbation theory. We compare our method to state-of-the-art methods for planning in GMDPs and show, that for non-local reward functions state-of-the-art methods approach the performance of random guess, while our method performs well. In the future, we want to use this planning method as the basis for a new reinforcement algorithm for multiple agents on a graph. 
\section*{Acknowledgements}
We thank the anonymous reviewers for helpful comments on the previous version of this manuscript.
Dominik Linzner acknowledges funding by the European Union's Horizon 2020 research and innovation programme (iPC--Pediatric Cure, No. 826121) and (PrECISE, No.  668858). Heinz Koeppl acknowledges support by the European Research Council (ERC) within the CONSYN project, No. 773196, and by the Hessian research priority programme LOEWE within the project CompuGene. 
\bibliography{cv_mdp_aaai}
\bibliographystyle{aaai}

\newpage
\onecolumn
\begin{appendices}

 \section{Appendix A -- KL-Divergence between two MDPs}
For any time-discretization, the KL-divergence takes the form 
\begin{align*}
KL(q\vert\vert p)=\sum_{X_{1},\dots X_{N},A_{1},\dots A_{N-1}}q(A_{1},\dots A_{N-1},X_{1},\dots X_{N})\ln \left[\frac{q(A_{1},\dots A_{N-1},X_{1},\dots X_{N})}{ p(A_{1},\dots A_{N-1},X_{1},\dots X_{N})}\right].
\end{align*}
Making use of the Markov property of both distributions
\begin{align*}
p/q((A_{1},X_{1}),\dots(A_{N},X_{N}))=p/q(A_{0},X_{0})\prod_{k=1}^{N}p/q((A_{k+1},X_{k+1})\mid(A_{k},X_{k})),
\end{align*}
we arrive after some  basic algebraic manipulations at 
\begin{align*}
KL(q\vert\vert p)=\sum_{k=1}^{N}\sum_{X_{k}A_{k}X_{k+1}}q(X_{k+1},A_{k},X_{k})\ln\left[ \frac{q(A_{k}\mid X_{k})q(X_{k+1}\mid A_{k},X_{k})}{p(A_{k}\mid X_{k})p(X_{k+1}\mid A_{k},X_{k})}\right],
\end{align*}
 -- or in the notation of the main-paper
 \begin{align*}
KL(q\vert\vert p)=\sum_{t=1}^{N}\sum_{y,x,a}q(y,x\mid a;t)\pi_{q}(a\mid x)\ln\left(\frac{q(y\mid x,a;t)\pi_{q}(a\mid x)}{p(y\mid x,a;t)\pi_{p}(a\mid x)}\right),
\end{align*}
where we identified the stationary policies $\pi_p(a\mid x)\equiv p(A_{k}\mid X_{k})$ and $\pi_q(a\mid x)\equiv q(A_{k}\mid X_{k})$. 
 \section{Appendix B -- KL-Divergence between two continuous-time MDPs}
 In order to perform the continuous-time limit, we have to define an expansion of the variational distribution $q$ in some infinitesimal time-step $h$.
  \begin{align*}
&q(y,x,a;t)\equiv\delta_{x,y}{q}(x;t)\pi_q(a\mid x)+h\frac{\tau(x,y,a;t)}{q(x;t)\pi_q(a\mid x)}+{o}({h}),
\end{align*}
and $\tau(x,x,a;t)=-\sum_{y\neq x}\tau(x,y,a;t)$. Plugging in this definition and the expansion $p(y\mid x,a)=\delta_{x,y}+h{w}(y\mid x,a)$, we can write
\begin{align*}
KL(q\vert\vert p)=\sum_{t}\sum_{y,x,a}\left[\delta_{x,y}q(x;t)+h\frac{{{w}}(x,y,a)}{\pi_{q}(a\mid x)}\right]\pi_{q}(a\mid x)\ln\left(\frac{\delta_{x,y}+h\frac{\tau(x,y,a;t)}{{q}(x;t)\pi_{q}(a\mid x)}}{\delta_{x,y}+h{{w}}(y\mid x,a)}\frac{\pi_{q}(a\mid x)}{\pi_{p}(a\mid x)}\right).
\end{align*}
Finally, after some algebraic manipulations, using $\lim_{h\rightarrow 0}\ln(1+hx)=hx$ and $\lim_{h\rightarrow 0}h\sum_{t=1}^T=\int_0^T\mathrm{d}t$
\begin{align*}
&KL(q\vert\vert p)=\int_0^T\mathrm{{d}t\;}\sum_{x,y\neq x,a}\left\{ {q}(x;t)\pi_{q}(a\mid x){{w}}(y\mid x, a)-\tau(x,y,a;t)\ln\left({{w}}(y\mid x, a)\pi_{p}(a\mid x)\right)\right\} \\
&+\int_0^T\mathrm{{d}t\;}\sum_{x,y\neq x,a}\tau(x,y,a)\left\{ \ln\left(\tau(x,y,a;t)\right)-\ln {q}(x;t)-\ln\left(\frac{\pi_{q}(a\mid x)}{\pi_{p}(a\mid x)}\right)-1\right\} \\
&+\frac{1}{h}\int_0^T\mathrm{{d}t\;}\sum_{x,y\neq x,a}{q}(x;t)\pi_{q}(a\mid x)\ln\left(\frac{\pi_{q}(a\mid x)}{\pi_{p}(a\mid x)}\right).
\end{align*}
If now the policies $\pi_q\neq \pi_p$, then the last term in the KL-divergence becomes infinite in the limit $h\rightarrow 0$. Thus we have to enforce $\pi=\pi_{q}=\pi_{p} $ and arrive at
\begin{align*}
KL(q\vert\vert p)=\int_0^T\mathrm{{d}t\;}\underset{\equiv E(t)}{\underbrace{\sum_{x,y\neq x,a}\left\{ {q}(x;t)\pi(a\mid x){{w}}(x,y\mid a)-\tau(x,y,a;t)\ln\left({{w}}(x,y\mid a)\pi(a\mid x)\right)\right\}}} \\
+\int_0^T\mathrm{{d}t\;}\underset{\equiv H(t)}{\underbrace{\sum_{x,y\neq x,a}\tau(x,y,a;t)\left\{ \ln\frac{\tau(x,y,a;t)}{{q}(x;t)}-1\right\} }}
\end{align*}
\section{Appendix C -- Discounting}
In order to incorporate discounting into the framework of planning via inference, one can introduce a prior over horizons $T\sim p(T\mid\gamma)$. In this case the KL-divergence between $q$ and the reward optimal posterior (see main-text) $p(X_{[0,\infty]},A_{[0,\infty]}\mid \pi, Z_{[0,\infty]}=1)$ becomes
\begin{align*}
&KL(q(X_{[0,T]},A_{[0,T]}\mid \pi)p(T\mid\gamma)\vert\vert p(T\mid\gamma)p(X_{[0,T]},A_{[0,T]}\mid \pi, Z_{[0,T]}=1))=\\
&KL(q(X_{[0,T]},A_{[0,T]}\mid \pi)p(T\mid\gamma)\vert\vert p(T\mid\gamma)p(X_{[0,T]},A_{[0,T]}\mid \pi))\\
&+\int_0^\infty \mathrm{d}T\,p(T\mid\gamma)\int_0^T\mathrm{d}t \sum_{s,a}q(s;t)\pi(a\mid s)R(s,a)-\ln p(Z_{[0,\infty]}=1\mid \pi).
\end{align*}
By using Fubinis theorem, we can exchange the integration order $\int_0^\infty \mathrm{d}T\,p(T\mid\gamma)\int_0^T\mathrm{d}t=\int_0^\infty \mathrm{d}t\int_t^\infty \mathrm{d}T\,p(T\mid\gamma)$ and further noticing that $\int_t^\infty \mathrm{d}T\,p(T\mid\gamma)=1-\int_0^t\mathrm{d}T\,p(T\mid\gamma)$, we arrive at the discount factor from the main text $d_\gamma(t)\equiv1-\int_0^t\mathrm{d}T\,p(T\mid\gamma)$. We notice that for an exponential prior $p(T\mid\gamma)=\ln\gamma \gamma^T$, we get $d_\gamma(t)=\gamma^t$ -- the standart exponential discount factor. We note, that the planning via inference framework allows naturally for non-exponential discounting, where in general a Bellmann equation can not be issued. Observing that $KL(q(X_{[0,T]},A_{[0,T]}\mid \pi)p(T\mid\gamma)\vert\vert p(T\mid\gamma)p(X_{[0,T]},A_{[0,T]}\mid \pi, Z_{[0,T]}=1))\geq 0$ we arrive at a variational lower bound to the marginal likelihood in the discounted case
\begin{align*}
\ln p(Z_{[0,\infty]}=1\mid \pi)\geq& \mathcal{F}[q,\pi]+V_q^\pi(s0),
\end{align*}
where we inserted the definition
\begin{align*} V_q^\pi(s_0)&=\sum_{s,a}\int_0^\infty \mathrm{d}t \, \gamma^t q(s;t)\pi(a\mid s)R(s,a)\\&=\mathsf{E}_q\left[\int_0^\infty \mathrm{d}t \gamma^t R(X(t),A(t))\mid X(0)=s_0,\pi \right] 
\end{align*}
and 
\begin{align*}
\mathcal{F}[q,\pi]\equiv-KL(q(X_{[0,T]},A_{[0,T]}\mid \pi)p(T\mid\gamma)\vert\vert p(T\mid\gamma)p(X_{[0,T]},A_{[0,T]}\mid \pi)).
\end{align*}
In a derivation, analogous to Appendix B above, we recover
\begin{align*}
\mathcal{F}[q,\pi]=-\int_0^\infty\mathrm{{d}t\;}d_\gamma(t)\left\{E(t)+H(t)\right\}.
\end{align*}
 \section{Appendix D -- Continuous-time variational lower-bound}\label{app:A}
\label{app:Cluster factorization}
In order to perform the continuous-time limit, we represent $q$ by an expansion in ${h}$ in set of marginals
 \begin{align*}
&q(y_n,x_n,u_n,a_n;t)=\delta_{x,y}{q}_n(x;t)q_n^u(t)\pi_n^u(a\mid x)+h\frac{\tau_n^u(x,y,a;t)}{{q}_n(x;t)q_n^u(t)\pi_n^u(a\mid x)}+{o}({h}),
\end{align*}
with $\tau^u_n(x,x,a,t)=-\sum_{y\neq x} \tau_n^u(x,y,a;t)$. 
By inserting $q's$ representation into $\mathcal{F}_{VPT}[q,\pi]$ we get
\begin{align*}
\mathcal{F}_{VPT}[q,\pi]&=\frac{1}{h}\int_0^\infty \mathrm{d}t d_\gamma(t)\,\sum_{y\neq x,x,u,a}{h}{\tau_{n}^{u}(x,y,a;t)}\left[\ln{h}\frac{{\tau_{n}^{u}(x,y,a;t)}}{q_n(x;t)q_n^u(t)\pi_n^u(a\mid x)}-\ln{h} \pi_n^u(a\mid x){{w}}^u_{n}(x,y\mid a)\right]\\
&+\sum_{x,u}\left\{ q_n(x;t)q_n^u(t)\pi_n^u(a\mid x)-{h}\sum_{y\neq x}{\tau_{n}^{u}(x,y,a;t)}\right\} \\
&\times \left[\ln\left\{ 1-{h}\frac{{\sum_{y\neq x}\tau_{n}^{u}(x,y,a;t)}}{q_n(x;t)q_n^u(t)\pi_n^u(a\mid x)}\right\}-\ln\left\{ 1-{h}\sum_{y\neq x,a}\pi_n^u(a\mid x){{w}}^u_{n}(y\mid x,a)\right\} \right]\\
\end{align*}
where we also inserted  ${P( X_n(t)=y_n\mid X_n(t)=x_n,U_n(t)=u_n,A_n(t)=a_n)=\delta_{x,y}+{ {{w}}^u_n(x,y\mid a)\pi_n^u(a\mid x)}{h}}$. With the asymptotic identity $\ln (1+{h}x)= {h}x$ we can simplify
\begin{align*}
\mathcal{F}_{VPT}[q,\pi]&=\frac{1}{h}\int_0^\infty \mathrm{d}t d_\gamma(t)\,\sum_{y\neq x,x,u,a}{h}{\tau_{n}^{u}(x,y,a;t)}\left[\ln\frac{{\tau_{n}^{u}(x,y,a;t)}}{q_n(x;t)q_n^u(t)\pi_n^u(a\mid x)}-\ln \pi_n^u(a\mid x){{w}}^u_{n}(x,y\mid a)\right]\\
&+\sum_{a,u,x}\left\{ q_n(x;t)q_n^u(t)\pi_n^u(a\mid x)-{h}\sum_{y\neq x}{\tau_{n}^{u}(x,y,a;t)}\right\} \\
&\times\left[ {h}\sum_{y\neq x}\pi_n^u(a\mid x){{w}}^u_{n}(x,y\mid a) -{h}\frac{{\sum_{y\neq x}\tau_{n}^{u}(x,y,a;t)}}{q_n(x;t)q_n^u(t)\pi_n^u(a\mid x)} \right]
\end{align*}which becomes in the continuous-time limit ${h} \rightarrow 0$ 
\begin{align*}
&\mathcal{F}_{VPT}[q,\pi]=\sum_{n}\int_0^\infty \mathrm{d}t d_\gamma(t)\,\underset{\equiv H_n(t)}{\underbrace{\sum_{x,y\neq x,u}{\tau_{n}^{u}(x,y,a;t)}[1-\ln{\tau_{n}^{u}(x,y,a;t)}+\ln(q_n^u(t){q_{n}(x;t)})]}}\\
+&\sum_{n}\int_0^\infty \mathrm{d}t d_\gamma(t)\,\underset{\equiv E_n(t)}{\underbrace{\sum_{y\neq x,x,a,u}\left[{q_{n}(x;t)}q_n^u(t){{{w}}^u_{n}(x,y\mid a)\pi_n^u(a\mid x)}+{\tau_{n}^{u}(x,y,a;t)}\ln{ {{w}}^u_{n}(x,y\mid a)\pi_n^u(a\mid x)}\right]}}.
\end{align*}
The contribution of the likelihood term can be derived to be
\begin{align*}
\mathsf{E}[R(s,a)]&=\sum_{n}\int_0^\infty \mathrm{d}t\, d_\gamma(t)\sum_{x,u}{q}_n(x;t)q_n^u(t) R^u_n(x,a)
\end{align*}

\section{Appendix E -- Approximate GMDP dynamics}\label{app:B}
\label{app:Euler-Lagrange}
We are now going to derive the dynamics of GMDPs, defined by fulfilling the Euler--Lagrange equations 
\begin{align*}
\partial_{x}\mathcal{{L}}[t,x,\dot{{x}}]-\partial_{t}[\partial_{\dot{{x}}}\mathcal{{L}}[t,x,\dot{{x}}]]=0.
\end{align*}
First lets consider the derivative with respect to ${{q}_n(x;t)}$:
\begin{align*}
\partial_{{{q}_n(x;t)}}H_n=d_\gamma(t)\sum_{u}\sum_{y\neq x}\frac{{\tau_n^u(x,y;t)}}{{{q}_n(x;t)}},\quad
\partial_{{{q}_n(x;t)}}E_j=d_\gamma(t)\mathsf{E}_{n}[{\sum_a {{w}}^u_n(x,x\mid a)\pi_n^u(a\mid x)}],
\end{align*}
Further if node $n$ has a child $j$
\begin{align*}
&\partial_{{{q}_n(x;t)}}H_j=d_\gamma(t)\sum_{x,u\mid X_n(t)=x_n=x}\sum_{y\neq x}\frac{\tau^u_j(x,y,t)}{{{q}_n(x;t)}},\\
&\partial_{{{q}_n(x;t)}}E_j=d_\gamma(t)\sum_{x}q_j(x;t)\mathsf{E}_n[{\sum_a {{w}}^u_n(x,x\mid a)\pi_n^u(a\mid x)}\mid X_n(t)=x].
\end{align*}
With respect to the derivative ${\dot{q}_n(x;t)}$ we get
\begin{align*}
\partial_{{\dot{q}_n(x;t)}}\mathcal{{L}}=-{{w}_n(x;t)}.
\end{align*}
We derive with respect to the transitions 
\begin{align*}
\partial_{{\tau_n^u(x,y,a;t)}}H_n=d_\gamma(t)\ln [{{q}_n(x;t)}q_n^u(t)]-\ln {\tau_n^u(x,y,a;t)},\quad
\partial_{{\tau_n^u(x,y,a;t)}}E_n=d_\gamma(t)\ln  {{{w}}^u_n(x,y\mid a)\pi_n^u(a\mid x)}.
\end{align*}
thus
\begin{align*}
\partial_{{\tau_n^u(x,y,a;t)}}\mathcal{L}=d_\gamma(t)\ln [{{q}_n(x;t)}q_n^u(t)]-d_\gamma(t)\ln {\tau_n^u(x,y;t)}+d_\gamma(t)\ln {{{w}}^u_n(x,y\mid a)\pi_n^u(a\mid x)}-{\eta_n(x;t)}+{\eta_n(y;t)}.
\end{align*}
The derivative with respect to the Lagrange-multipliers yields:
\begin{align*}
\partial_{{\eta_n(x;t)}}\mathcal{L}=-\left\{{\dot{q}_n(x;t)}-\left[\sum_{y\neq x,u}{\tau_n^u(y,x;t)}-{\tau_n^u(x,y;t)}\right]\right\}
\end{align*}
And lastly derivatives of $\mathsf{E}[R(s,a)]]$
\begin{align*}
\partial_{{{q}_n(x;t)}}\mathsf{E}[R(s,a)]=d_\gamma(t)\sum_{u,a}q_n^u(t)\pi_n(a\mid x)R^u_n(x,a)+d_\gamma(t)\sum_{j\in\mathrm{child}(n)}\sum_{u,a}q_j(x;t)q_j^{u/n}(t)\pi_n(a\mid x) R^u_j(x,a)
\end{align*}

These can then be combined as the following Euler-Lagrange equations:
\begin{align*}
&(\mathrm{I})\quad0=d_\gamma(t)\sum_{u}\sum_{y\neq x}\frac{{\tau_n^u(x,y,a;t)}}{{{q}_n(x;t)}}+d_\gamma(t)\mathsf{E}_{n}[{ {{w}}^u_n(x,y\mid a)\pi_n^u(a\mid x)}]+\dot{\eta}_n(x;t)+d_\gamma(t)\mathsf{{E}}_{n}[{R^u_n(x,a)}]\\
&\quad\quad\quad+d_\gamma(t)\sum_{j\in \mathrm{child}(n)}\sum_{x,u\mid X_n(t)=x}\sum_{y\neq x}\frac{{\tau_n^u(x,y,a;t)}}{{{q}_n(x;t)}}\\
&\quad\quad\quad+d_\gamma(t)\sum_{x}q_j(x;t)\left\{\mathsf{E}_{j}[{{w}}^u_j(x,x\mid a)\mid X_n(t)=x]+\mathsf{E}_{j}[{R^u_j(x,a)}\mid x_{i}=x]\right\}\\
&(\mathrm{II})\quad 0=\ln [{{q}_n(x;t)}q_n^u(t)]-\ln {\tau_n^u(x,y,a;t)}+\ln {{{w}}^u_n(x,y\mid a)\pi_n^u(a\mid x)}-{\eta_n(x;t)/d_\gamma(t)}+{\eta_n(y;t)/d_\gamma(t)}\\
&(\mathrm{III})\quad{\dot{q}_n(x;t)}=\sum_{y\neq x,u,a}\left\{\tau_n^u(y,x,a;t)-\tau_n^u(x,y,a;t)\right\}.
\end{align*}
Exponentiating $(\mathrm{II})$ gives
\begin{align*}
(\mathrm{II}^*)\quad{\tau_n^u(x,y,a;t)}={{q}_n(x;t)}q_n^u(t) { {{w}}^u_n(x,y,a\mid a)\pi_n^u(a\mid x)}{\rho_n(y;t)}/{\rho_n(x;t)},
\end{align*}
where ${\rho_n(x;t)}\equiv\exp({\eta_n(x;t)/d_\gamma(t)})$. Assuming that ${{w}}$ is irreducible, ${\rho_n(x;t)}$ and ${{q}_n(x;t)}$ are non-zero in $(0,T)$
 and we can thus eliminate ${\tau_n^u(x,y,a;t)}$ in $(\mathrm{I})$ and $(\mathrm{II})$.
 Thus
 \begin{align*}
(\mathrm{I}^*)\quad{\dot{\rho}_n(x,t)}&=\sum_{y\neq x,a}\mathsf{E}_{n}[{{{w}}^u_n(x,y\mid a)\pi_n^u(a\mid x)}]{\rho_n(y;t)}\\
+&\left\{\mathsf{E}_{n}[{ {{w}}^u_n(x,x\mid a)\pi_n^u(a\mid x)}]+\psi_n(x;t)+\ln\rho_n(x;t)\frac{\partial_t d_\gamma(t)}{d_\gamma(t)}\right\}{\rho_n(x;t)}\\
(\mathrm{III}^*)\quad{\dot{q}_n(x;t)}&=\sum_{y\neq x,a}\left\{m_n(y)\mathsf{E}_{n}[{{{w}}^u_n(y,x\mid a)\pi_n^u(a\mid y)}]{\rho_n(x;t)}/{\rho_n(y;t)}\right.\\
-&\left.{{q}_n(x;t)}\mathsf{E}_{n}[{ {{w}}^u_n(x,y\mid a)\pi_n^u(a\mid x)}]{\rho_n(y;t)}/{\rho_n(x;t)}\right\},
 \end{align*}
 where we used that
 \begin{align*}
\frac{\partial_{t}\eta_{i}(x)}{d_\gamma(t)}= \frac{1}{\rho_{i}(x;t)}\partial_{t}\rho_{i}(x;t)+\ln\left(\rho_{i}(x;t)\right)\frac{\partial_{t}d_\gamma(t)}{d_\gamma(t)}
 \end{align*}. We further summarized
 \begin{align*}
 \psi_n(x;t)=&\sum_{j\in\mathrm{child}(n)}\sum_{x}q_j(x;t)\left\{ \sum_{y\neq x,a}\frac{\rho_j(y;t)}{\rho_j(x;t)}\mathsf{E}_{j}[{{w}}^u_j(x,y\mid a)\pi_j^u(a\mid x)\mid X_n(t)=x]\right.\\
 +&\left.\mathsf{E}_{j}[{{w}}^u_j(x,x\mid a)\pi_j^u(a\mid x)\mid X_n(t)=x]+\mathsf{{E}}_{j}[{C_j(x,u)}\mid x_{i}=x]\right\}.
 \end{align*}
 
\section{Appendix F -- Policy evaluation}
 \begin{figure}[t!]
	\begin{center}
		\includegraphics[width=0.5\columnwidth]{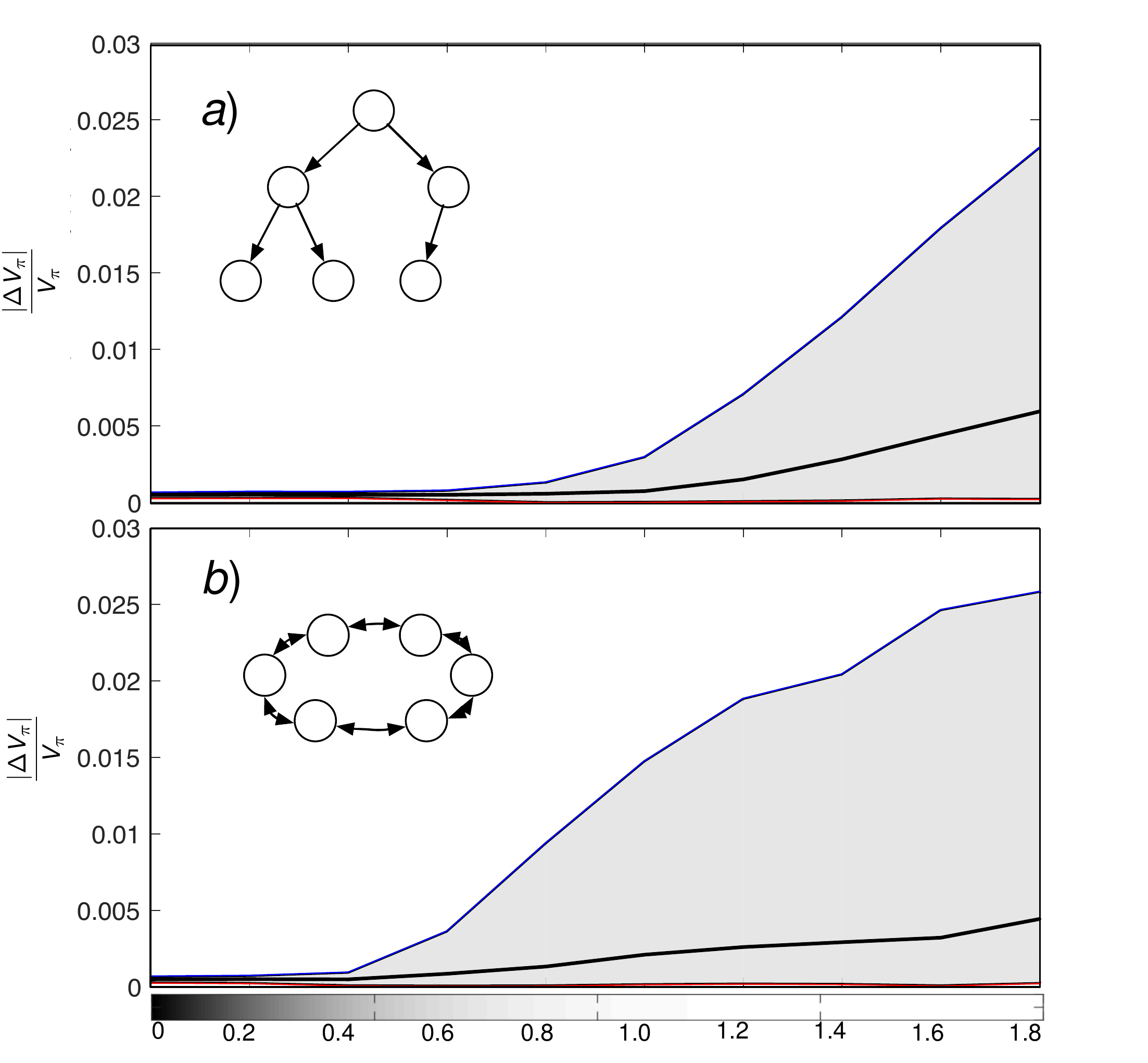}
       \end{center}
	\caption{Relative deviation of the approximate expected reward from the exact one for different policies parametrized by $\beta$. For reference we plotted the scale of $\pi \propto 2\tanh (\beta)$ on the x-axis (colormap: white$=0$ black$=1$). The performance is evaluated using 50 random reward functions of type (10) for two different graph topologies (inset). We plotted the $5\%$ (red), $50\%$ (black) and $95\%$ (blue) percentiles. We find that our method performs slightly better on trees a) than on bi-directed chains b).}
	\label{bound}
\end{figure}

We want to test how accurately we can approximate the true lower bound using Prop. 2. We define local state and action spaces $\mathcal{X}=\{0,1\}$ and $\mathcal{A}=\{0,1\}$. To test the performance under different types of policies, we consider policies 
\begin{align}\label{ising_pol}
 \pi^u_i(a\mid x)=\frac{1}{2}+\mathrm{tanh}\left(\beta x_i\sum_{j\in\mathrm{par(i)}}u_j\right),
 \end{align}
that become increasingly deterministic with increasing $\beta$. In order to keep this experiment simple, we assume a deterministic transition model ${{{w}}^u_n(y\mid x,a)=\delta_{x,(-1)^a}}$ for $y\neq x$, independent on the parent configuration. Because in this experiment, the agents are coupled by their policies, a more deterministic policy corresponds to a stronger coupling between the agents, thus increasing the perturbation parameter $\varepsilon$. We test our method on two different tree topologies and random reward functions of type (10). We set the variances ${\sigma_\alpha=\sigma_J=0.2}$. One is is a tree-network sketched in the inset of Fig. \ref{bound} a), the other a bi-directional chain with periodic boundary conditions as sketched in Fig. \ref{bound} b). We find that for both networks, the accuracy of the approximation lowers for increasing $\beta$, however the accuracy is worse for the bi-directional chain. 
\end{appendices}

\end{document}